%% file: Neurips_Final_Submission.tex
\documentclass{article}

\usepackage[main,final]{neurips_2026}
\PassOptionsToPackage{numbers,sort&compress}{natbib}
\usepackage[utf8]{inputenc} 
\usepackage[T1]{fontenc}    
\usepackage{hyperref}       
\usepackage{url}            
\usepackage{booktabs}       
\usepackage{amsfonts}       
\usepackage{nicefrac}       
\usepackage{microtype}      
\usepackage{xcolor}         

\newcommand{\R}{\mathbb{R}}
\newcommand{\E}{\mathbb{E}}
\def\ceil#1{\lceil #1 \rceil}

\usepackage{enumerate}
\usepackage{graphicx}
\graphicspath{{./}{../}}
\usepackage{xcolor}

\usepackage{mathrsfs,mathtools,amssymb}

\usepackage{algorithm} 
\usepackage{algpseudocode} 
\usepackage[center, font=small,labelfont=sc,labelsep=none]{caption}

\usepackage{float}
\usepackage{amsthm}

\usepackage[capitalize,nameinlink,noabbrev]{cleveref}

\crefname{section}{Section}{Sections}

\Crefname{section}{Section}{Sections}

\crefname{lemma}{Lemma}{Lemmas}

\Crefname{lemma}{Lemma}{Lemmas}

\crefname{prop}{Proposition}{Propositions}
\Crefname{prop}{Proposition}{Propositions}
\crefname{claim}{Claim}{Claims}
\Crefname{claim}{Claim}{Claims}
\crefname{fact}{Fact}{Facts}
\Crefname{fact}{Fact}{Facts}

\renewcommand{\max}{\mathsf{max}}

\renewcommand{\S}{\mathbb{S}}

\newcommand{\N}{\mathbb N}

\newcommand{\eps}{\varepsilon}
\newcommand{\ip}[2]{\langle #1,#2 \rangle}
\newcommand{\conv}{\mathsf{conv}}
\newcommand{\f}{\mathsf{m}}

\DeclareMathOperator{\relu}{\mathsf{ReLU}}
\DeclareMathOperator{\ICNN}{\mathsf{ICNN}}
\DeclareMathOperator{\CPWL}{\mathsf{CPWL}}

\DeclareMathOperator{\MAX}{\mathsf{MAX}}
\DeclareMathOperator{\cP}{\mathcal{P}}

\newtheorem{theorem}{Theorem}
\newtheorem{lemma}[theorem]{Lemma}

\newtheorem{claim}[theorem]{Claim}
\newtheorem{prop}[theorem]{Proposition}

\newtheorem{corollary}[theorem]{Corollary}
\newtheorem{fact}[theorem]{Fact}

\theoremstyle{definition}

\theoremstyle{remark}

\newtheorem{remark}[theorem]{Remark}

\title{On the depth of monotone R\lowercase{e}LU neural networks and ICNN\lowercase{s}}

\author{%
  Egor Bakaev \\
  University of Copenhagen \\
  \texttt{egba@di.ku.dk} \\
  \And
  Florestan Brunck \\
  University of Copenhagen \\
  \texttt{flbr@di.ku.dk} \\
  \And
  Christoph Hertrich \\
  University of Technology Nuremberg \\
  \texttt{christoph.hertrich@utn.de} \\
  \AND
  Daniel Reichman \\
  Worcester Polytechnic Institute \\
  \texttt{daniel.reichman@gmail.com} \\
  \And
  Amir Yehudayoff \\
  University of Copenhagen \\
  Technion--Israel Institute of Technology \\
  \texttt{amye@di.ku.dk} \\
}

\begin{document}

\maketitle


\begin{abstract}
We study two models of $\relu$ neural networks: monotone networks ($\relu^+$)
and input convex neural networks ($\ICNN$).
Our focus is on expressivity, mostly in terms of depth, and our motivation is to gain a better understanding
of the depth required to exactly represent functions using neural networks with ReLU activations, a subject that has received a lot of attention lately.
We prove several lower bounds on the depth required to represent functions by monotone ReLU networks and $\ICNN$.
For the maximum function $\MAX_n$ computing the maximum of $n$ real numbers,
we show that $\relu^+$ networks
cannot compute $\MAX_n$,
or even approximate it.
We prove a sharp lower bound of $n-1$ on the $\ICNN$ depth complexity of $\MAX_n$.
We also prove depth separations between
$\relu$ networks and $\ICNN$s;
for every~$k$, there is a depth-$2$ $\relu$ network of size $O(k^2)$ that cannot be simulated by a depth-$k$ $\ICNN$.
The proofs are based on deep connections between neural networks and polyhedral geometry, and also use isoperimetric properties of triangulations.
\end{abstract}

\maketitle

\section{Introduction}

Neural networks form an important computational model
because of their many applications.
The gates in a neural network, generally speaking,
perform linear operations followed by non-linear operations.
A standard non-linearity is the rectified linear unit ($\relu$) defined by
$\relu(x) = \max \{0,x\}$.
$\relu$ networks form a central family of neural networks (see~\citep{nair2010rectified,glorot2011deep,hertrich2021towards,haase2023lower,valerdi2024minimal} and the many references within).

There are two categories of high-level questions concerning neural networks:
``dynamic'' and ``static''.
Dynamic questions are about the behavior of the neural network during the training process, and their generalization capabilities.
Static questions are about expressivity and computational power.
Our focus is on the static, computational complexity aspects. 
There are several basic challenges in understanding the expressivity 
of ReLU networks (see~\citep{arora2018understanding,williams2018limits,
 hertrich2021towards,haase2023lower,valerdi2024minimal}
and references within). Following the success of deeper (with dozens of layers) architectures in applications, there has been extensive study of the benefits of depth~\citep{eldan2016power,daniely2017depth,safran2022depth,telgarsky2016benefits,arora2018understanding} in terms of the \emph{expressive power} of neural networks.
Our focus is on understanding depth as a computational resource for \emph{exactly} representing functions in different models of $\relu$ networks. A better understanding of these models can lead to a better understanding of the role of architectural restrictions in $\relu$ networks. This topic has received significant attention lately~\citep{arora2018understanding,haase2023lower,averkov2025expressiveness,bakaev2026better,ruess2026shallower,brunck2026two}. 

Let us introduce some notation.
An \emph{affine} function is of the form $\R^m \ni x \mapsto \ip{a}{x}+b$ with $a \in \R^m$ and $b \in \R$;
the number $b$ is called the \emph{bias} term. 
If the bias term is zero, the function is called \emph{linear}. 
The inputs we work with are $x=(x_1,\ldots,x_n) \in \R^n$.
A $\relu$ network can be represented as a directed acyclic graph whose input gates are the $x_i$'s, and the inner gates  compute either an
affine function or the $\relu$ operation (see \Cref{fig:neural}).\footnote{Note that our definition of a neural network as directed acyclic graph explicitly allows skip connections. This can make a crucial difference in terms of expressivity, in particular when considering monotone models as we do in this paper.}
The depth of a $\relu$ network is the maximum number of $\relu$ gates in a directed path in it.
Note that this differs from the usage of the word ``depth'' in the majority of the literature about ReLU network expressivity, where the depth is defined as this quantity plus one. There, our notion of depth is
usually called ``the number of hidden layers''.
For a depth parameter $k\in \mathbb{N}$,
we define:
\begin{align*}
	\relu_{n,k} &:=\{f\colon\R^n\to\R : \text{$f$ is computable by a depth-$k$ $\relu$ network}\}.
\end{align*}
Affine functions belong to
$\relu_{n,0}$ and
$\relu_{n,k}\subseteq \relu_{n,k+1}$.
We require \emph{exact representation}; namely, for every function $f$ in $\relu_{n,k}$, there exists a $\relu$ network of depth $k$ that is equal to $f$ for every possible input in $\R^n$. Also, observe that no restriction is placed on the \emph{size} of the network computing $f$, other than that it is finite.

\begin{figure}\centering
\includegraphics[page=17]{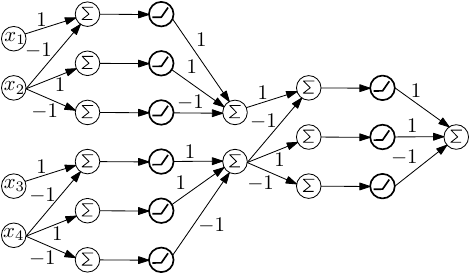}
\caption{. A depth $2$ $\relu$ neural network computing the maximum of $4$ elements.}
\label{fig:neural}
\end{figure}

When the depth $k$ is not important, we denote
by $\relu_n$ the union
$$ \relu_n := \bigcup_{k=1}^\infty \relu_{n,k}$$
and use a similar convention for the classes defined below.
$\relu$ networks compute continuous
and piecewise affine functions:\footnote{The notation $\CPWL$ is standard;
the ``L'' suggests ``linear'' but in fact the meaning is ``affine''. In this text, we always assume that the number of pieces is finite.} 
\begin{align*}
    \CPWL_n&:=\{f\colon\R^n\to\R : \text{$f$ is continuous and piecewise affine}\}.
\end{align*}

It is common in computational complexity theory that
the analysis of the general model is difficult and the main questions are open. This leads to the study of restricted models. For example, a monotone restriction has been studied in a variety of circuit models such as boolean, algebraic and threshold circuits; see e.g.~\citep{chattopadhyay2021lower,raz2011multilinear,pudlak1997lower,alon1987monotone,haastad1991power,robere2016exponential}.
We shall study two families of restricted networks, monotone $\relu$ networks ($\relu^+$s) and input convex neural networks ($\ICNN$s), which we define next.

The space $\R^m$ is partially ordered via
$x \leq y$ iff $x_i \leq y_i$ for all $i$.
A function $F : \R^n \to \R$ is \emph{monotone} if 
$F(x) \leq F(y)$ for all $x \leq y$. 
An affine function $x \mapsto \ip{a}{x}+b$ is monotone iff $a \geq 0$; the bias term can be an arbitrary real number.
A {\em monotone $\relu$ network} is a $\relu$ network in which every affine gate computes a monotone function. 
An $\ICNN$ is the same as a monotone $\relu$
network, except that gates that compute affine functions of the {\em inputs} $x_1,\ldots,x_n$
are not restricted to be monotone
(and all other affine gates
are restricted to be monotone).
In other words, the only gates that are allowed to be non-monotone in an $\ICNN$ are before the first $\relu$ gates. 
We use the notation
\begin{align*}
	\relu^+_{n,k} &:=\{f\colon\R^n\to\R :\text{$f$ is computable by a depth-$k$ monotone $\relu$ network}\},\\
    \ICNN_{n,k} &:=\{f\colon\R^n\to\R :\text{$f$ is computable by a depth-$k$ $\ICNN$ network}\} .
\end{align*}
It is straightforward that every function in $\relu^+_{n}$ is $\CPWL$, monotone and convex.
Similarly, every function in $\ICNN_{n}$ is $\CPWL$ and convex.
It is also obvious that for every $k$, we have that
$\relu^+_{n,k} \subset
\ICNN_{n,k}$ with strict inclusion.

\textbf{Motivation and Applications.} Monotone models of neural networks (where the weights of the neurons are nonnegative) were suggested, motivated and studied in~\citep{daniels2010monotone,sivaraman2020counterexample,mikulincer2024size,hertrich2024neural}.
$\ICNN$s were introduced, motivated and studied in~\citep{amos2017input}.
Theoretically, they serve
as a model for studying $\relu$ networks that compute convex functions. 
In practice, preserving convexity with respect to the inputs enables the use of convex optimization methods when optimizing over those inputs. 
For these reasons, $\ICNN$s were subsequently studied in many theoretical and applied works~\citep{bunning2021input,chen2018optimal,chen2020data,bunne2023learning,hertrich2024neural,froese2026parameterized,das2026parameterized,hertrich2027tropical}.
For example, they have found use in optimal transport and predicting cellular responses, electrical grid regulation, model predictive control, reinforcement learning and image reconstruction.

We believe the study of monotone neural networks is interesting for several reasons. 
First, the simplicity of the monotone model allows to expose more structure,
which can potentially highlight the steps we need to take 
in order to understand the general model.
Second, in the closely related model of maxout networks (where each
gate computes the maximum of affine functions of its inputs), a general
network can be written as the difference of two monotone maxout networks
of the same size~\citep{hertrich2024neural,brandenburg2026decomposition}.
Although this decomposition does not extend to monotone ReLU networks,
it illustrates how the study of monotone computational models can
provide insights into unrestricted neural networks.
In addition, the monotone setting leads to interesting geometric
definitions and questions.
One example is the difference between 
$\R^2$ and $\R^3$ 
exhibited by \Cref{prop:isoPlane}
and \Cref{prop:example}.

\section{Monotone networks}
For a broad family of activation functions (including $\relu$s), the universal approximation theorems say that neural networks of depth one can \emph{approximate} any continuous function over a bounded domain; see e.g.~\citep{cybenko1989approximation,hornik1989multilayer} and references within. The minimum depth needed to compute arbitrary $\CPWL$ functions exactly was open until recently. It is now known that every $\CPWL$ function can be computed by a depth-two $\relu$ network~\citep{brunck2026two}.

A central function in this area is the maximum function $\MAX_n \in \CPWL_n$ defined by
$$\MAX_n(x) = \max \{x_1,x_2,\ldots,x_n\}.$$ 
There are many reasons to study $\MAX_n$.
Most importantly, it is ``complete'' for the class of all $\CPWL$ functions as we explain next. 
\citet{wang2005generalization}
showed that every function in $\CPWL_n$
can be written as a linear combination
of $\MAX_{n+1}$ functions applied to some  affine functions; see also~\citep{arora2018understanding,hertrich2021towards}.
In particular,  if $\MAX_{n+1} \in \relu_{n+1,k}$
then $$\CPWL_n \subseteq \relu_{n,k}.$$ In words, the depth complexity of $\MAX_{n+1}$ is essentially equal to the depth complexity of all of $\CPWL_n$.

The standard divide-and-conquer construction shows that the depth complexity of $\MAX_n$ is at most $\ceil{\log_2 n}$. It follows that $\CPWL_n \subseteq \relu_{n,k}$ with $k = \ceil{\log_2 (n+1)}$.
While \citet{hertrich2021towards} conjectured that this upper bound was tight, the conjecture was refuted in \citep{bakaev2026better}, where it was proved that $\MAX_5 \in \relu_{5,2}$. This led to the improved bounds $\ceil{\log_3(n-2)}+1$ for $\MAX_n$ and $\ceil{\log_3(n-1)}+1$ for all functions in $\CPWL_n$. \citet{ruess2026shallower} subsequently proved that $\MAX_n \in \relu_{n,2}$ for every $n \leq 12$, and improved the two bounds to $\ceil{\log_6(n/2)}+1$ and $\ceil{\log_6((n+1)/2)}+1$, respectively. Most recently, \citet{brunck2026two} proved that
$$\CPWL_n = \relu_{n,2}$$
by showing that $\MAX_n \in \relu_{n,2}$ for every $n$.

The representations of $\MAX_n$ above use unrestricted rational weights. Under additional restrictions, stronger depth lower bounds are known. For example, the depth complexity of $\MAX_n$ is $\ceil{\log_2 n}$ when the weights of the neurons are integral~\citep{haase2023lower}, matching the upper bound by \citet{arora2018understanding} and confirming the conjecture by \citet{hertrich2021towards} in this case. Other related lower bounds were shown in \citep{averkov2025expressiveness,grillo2025depth}.

Here we study the depth requirement of functions that can be exactly represented by $\relu^+$ networks as well as $\ICNN$s. 
We start by investigating $\MAX_n$, which is monotone and convex.
Can monotone ReLU networks compute $\MAX_n$?

\begin{claim}
$\MAX_2 \not \in \relu^+_{2}$.
\end{claim}

The claim easily follows from the
fact that the non-differentiable points of $\MAX_2$ is the line
$x_1-x_2 = 0$, whose normal 
is not monotone
but on the other hand,
the non-differentiable points
of every function in $\relu^+_{2}$ have monotone normals. 
We shall not provide a full proof, because we shall prove stronger statements below. 

The question of approximating $\MAX_n$ by a $\relu$ network was studied in~\citep{safran2024many}, where they showed it can be done with a $\relu$ network of depth two and size $O(n^2)$. However, approximating $\MAX_n$ with a monotone $\relu$ network was not previously studied. Can monotone $\relu$ networks even approximate $\MAX_n$? The following theorem shows that they cannot.

\begin{theorem}
There is $\eps > 0$ so that the following holds. 
For every $F \in \relu^+_{2}$,
there is $r>0$ 
so that if $x \in [0,r]^2$
is chosen uniformly at random then
$$\E |F(x)-\MAX_2(x) | > \eps.$$\end{theorem}

\begin{proof}[Proof sketch]
Fix $F \in \relu^+_{2}$.
Let $r' > 1$ be large enough
so that every gate in the network for $F$
computes an affine function
on inputs from $[r',\infty)^2$ ($r'$ may depend on the weights of $F$). For $r=2r'$,
the function $\MAX_2$ is far from any affine function on $[r',r]^2$.
\end{proof}

In the previous theorem, the domain of inapproximability depends on the specific network we consider. 
Can monotone ReLU networks approximate $\MAX_n$ over $[0,1]^n$?
Let us look on the plane for example. The domain $[0,1]^2$ can be partitioned to pieces (in fact, triangulated) with monotone normals\footnote{The normals could be for example $(1,0)$, $(0,1)$ and $(1,1)$.} 
so that $\MAX_2$ is approximated by 
a $\CPWL$ function with this pieces.
In fact, any continuous function on $[0,1]^2$ can be approximated in this way. 

So, the argument above does not seem to imply that
$\MAX_2$ cannot be approximated by a monotone $\relu$ network in the domain $[0,1]^2$.
To prove that it cannot, we identify an additional structure of monotone $\relu$ networks.

A map between two partially ordered sets is called isotonic if
it preserves the order.
For a convex map $F : \R^n \to \R$ denote by $\partial F(x)$ 
the sub-gradient of $F$ at $x \in \R^n$:
$$\partial F(x) = \{ g \in \R^n : \forall y \in \R^n \ F(y) \geq F(x) + \ip{g}{y-x}\};$$
see \Cref{fig:newton} for an example.
Because $F$ is convex, the sub-gradient $\partial F(x)$ is non-empty and convex for all $x$.
If $F$ is differentiable at $x$ then $\partial F(x)= \{ \nabla F(x) \}$,
where $\nabla F$ is the gradient.

For two sets $A, B \subset \R^n$,
we write $A \leq B$ if for every $a \in A$, there is $b \in B$ so that
$a \leq b$, and vice versa
(for every $b \in B$,
there is $a \in A$ so that $a \leq b$).
It follows that $\leq$ is transitive
and if $A_1 \leq B_1$
and $A_2 \leq B_2$
then $A_1+A_2 \leq B_1 + B_2$,
where $+$ denotes Minkowski sum.
We say that the gradient of $F$ is {\em isotonic} if 
$$\forall x \leq y \ \ \partial F(x) \leq \partial F(y).$$
We say that the gradient of $F$ is non-negative if $\partial F(x) \geq \{0\}$ for all $x$.

In the way we set things up, functions with isotonic gradients are always convex (because we need sub-gradients).
The notion of isotonic gradients, however, makes sense for differentiable functions as well (with gradients instead of sub-gradients).
In dimension one, for a differentiable function $F$,
the function $F$ is convex iff $F$ has isotonic gradients. 
In dimension two or higher, this is no longer true; for example, the function $(x_2-x_1)^2$ on $[0,1]^2$ is convex but does not have isotonic gradients, and the function $x_1 \cdot x_2$ is not convex but it does have isotonic gradients. 

The structure of monotone $\relu$ networks we identify is summarized in the following lemma.

\begin{lemma}
\label[lemma]{lem:isoto}
If $F \in \relu^+_{n}$,
then the gradient of $F$ is isotonic and non-negative. 
\end{lemma}

This structure allows (as the next theorem shows) to deduce that monotone $\relu$ networks can not even approximate the maximum function.
In particular, there is no sequence of monotone $\relu$ networks 
that tend to the maximum function
even in the unit square
(a similar statement holds in higher dimensional space; we focus on the plane for simplicity).

\begin{theorem}
\label{thm:noApp}
There is a constant $\eps > 0$ so that the following holds.
Let $F \in \CPWL_{2}$ be convex with isotonic gradient. 
Let $x$ be uniformly distributed in $[0,1]^2$.
Then,
$$\E |F(x) - \MAX_2(x)| > \eps.$$
\end{theorem}

\textit{Proof Sketch.} The key insight is that the behaviour of the gradient of the $\MAX_2$ function is incompatible with isotonicity.
Since below the main diagonal (for $x_1>x_2$), its gradient is $(1,0)$, while it is $(0,1)$ above it, we can find two points $x\leq y$ whose gradients are not ordered coordinatewise.
We then show that any convex function approximating $\MAX_2$ sufficiently well necessarily exhibits a similar behavior: small approximation error forces its gradients to be close to those of $\MAX_2$ at suitable points.
This contradicts isotonicity and yields a lower bound on the approximation error.

A complete proof of \Cref{thm:noApp} and \Cref{lem:isoto} is given in Appendix~\ref{sec:Inapprox}.

We now know that every function that is computed by a monotone $\relu$
network is (1) monotone, (2) convex and (3) has isotonic gradient.
These three properties are necessary for being computed by 
a monotone $\relu$ network.
Are these three conditions also sufficient?
The answer depends on the dimension $n$,
as the two next propositions show.

For simplicity, we focus on the homogeneous case.
A function $F : \R^n \to \R$
is homogeneous (also known as homogeneous of degree one) if for every $a \geq 0$ we have $F(ax) = a F(x)$.
For example, $\MAX_n$ is homogeneous. 
A $\relu$ network is \emph{homogeneous} if all affine functions in it
are linear (i.e., all bias terms are zero).
We may assume that a $\relu$ network
computing a homogeneous function $F$ is, without loss of generality, homogeneous
(see e.g.~\citep{hertrich2021towards}).
This is due to the fact that homogeneous functions are determined by their large scale behavior, and biases become arbitrarily negligible on scaled inputs.
Thus given a $\relu$ network which computes a homogeneous function, we can show that the same network with all bias terms set to zero computes the same function.

\begin{prop}
\label[prop]{prop:isoPlane}
For $n = 2$,
if $F \in \CPWL_n$ is homogeneous, monotone, convex and
with isotonic gradient then $F \in \relu^+_{n,2}$.
\end{prop}

\begin{prop}
\label[prop]{prop:example}
For every $n \geq 3$, there is a homogeneous, monotone and convex $F \in \CPWL_n$ 
with isotonic gradient so that $F \not \in \relu^+_n$.
\end{prop}

\begin{figure}\centering
\includegraphics[page=9]{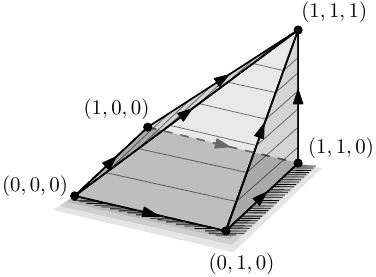}
\caption{. The square pyramid $P_*$.}
\label{fig:pyramid}
\end{figure}

\textbf{Proof Sketch.}
The proof relies on the duality between homogeneous convex functions and their Newton polytopes (see \cref{sec:NewP}).
In this new geometric language, isotonicity of the gradient then corresponds to a simple geometric condition: the edges of the polytope can be oriented in nonnegative directions.
In dimension two, we observe that every polygon can be decomposed into a Minkowski sum of segments and triangles. For a polygon with positive edges in the nonnegative orthant, these summands can also be chosen to lie in the nonnegative orthant and have positive edges, so each can be represented by a monotone $\relu$ network of depth at most two.
In contrast, in dimension three, we can construct a square pyramid (see \Cref{fig:pyramid}) whose edges satisfy the same positivity condition but which cannot be generated by the operations corresponding to monotone $\relu$ networks.
The obstruction ultimately comes from the notion of indecomposability under Minkowski sums and the impossibility of constructing it by joining with the origin through a $\relu$ operation.

The proofs of \Cref{prop:isoPlane} and \Cref{prop:example} are given in Appendix~\ref{sec:NewP}.
\Cref{prop:example} also yields a non-homogeneous counterexample by taking $G=F+1$.
By construction, this function is monotone, convex and retains isotonic gradients.
If $G$ belonged to $\relu^+_n$, then subtracting one through the output bias would give a monotone $\relu$ network representing $F$, contradicting the proposition.
The reason we focus on the homogeneous case is that the theory on Newton polytopes developed in \Cref{sec:NewP} is cleaner in this case.

It is an interesting question to characterize the family of functions that can be represented exactly by $\relu^+_n$. We leave this question for future work.

As described above, the $\MAX_n$ function cannot be computed or even approximated by networks in $\relu^+_n$ regardless of their depth. One may wonder whether it is possible to show benefits of depth for functions that \emph{can} be computed by networks in $\relu^+_n$.  
An analogous phenomenon was studied in the context of monotone boolean circuit complexity. Following the super-polynomial lower bound for the monotone circuit complexity of the clique function~\citep{razborov1985lower_boolean}---which is believed to require super-polynomial size of arbitrary circuits---several works demonstrated the existence of monotone boolean functions that \emph{can} be computed by boolean circuits of \emph{polynomial size}, but nevertheless require monotone boolean circuits of super-polynomial size~\citep{razborov1985lower_permanent,cavalar2023constant,tardos1988gap,alon1987monotone}.
A similar statement was proved in the algebraic setting~\citep{valiant1979negation}.

We show that depth that is linear in the input dimension can be crucial for the computation of functions in $\relu^+_n$.
Towards this end we inductively define the functions
$$\f_0 = 0$$
and for $n>0$,
$$\f_n(x) = \relu(x_n + \f_{n-1}(x_1,\ldots,x_{n-1})).$$
It seems worth noting that the function $\f_n$ corresponds
to the so-called
Schl\"afli orthoscheme (see \Cref{sec:Monotone} for more details
).
We compute the $\relu^+$
depth complexity of $\f_n$.

\begin{theorem}
\label{thm:LB-f}
For every $n > 0$,
$$\f_n \in \relu^+_{n,n}$$
but for every $k<n$,
$$\f_n \not \in \relu^+_{n,k}.$$
\end{theorem}

\textit{Proof Sketch.} The proof exploits the duality (see \cref{sec:NewP}) between $\f_n$ and a special $n$-dimensional simplex - the Schl\"afli orthoscheme (see \Cref{fig:orthoscheme}).
A central notion in the proof is again that of indecomposability: if a simplex can be expressed as a Minkowski sum, one of the summands must necessarily be a homothetic copy of the same simplex.
This means Minkowski sums cannot build this simplex from simpler polytopes.
The key observation is that the facet of this special simplex opposite the origin is itself a copy of the same lower dimensional orthoscheme, which corresponds to $\f_{n-1}$.
Working backwards through the network from the output to the input gates, taking this facet allows us to recover a representation of $\f_{n-1}$ with one fewer $\relu$ layer.
We repeat this argument inductively to show that at least $n$ layers are necessary.
We note that this bound is tight, as the recursive definition of $\f_n$ directly gives a monotone $\relu$ network of depth $n$.

Together with the fact that every $\CPWL$ function is in $\relu_{n,2}$~\citep{brunck2026two}, the theorem leads to the following separation between general $\relu$ networks
and monotone $\relu$ networks.

\begin{figure}
	\centering
	\includegraphics[page=15]{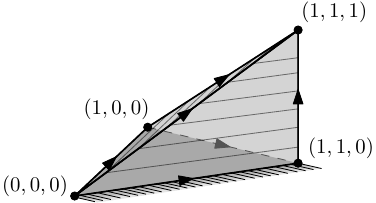}
	\caption{. The simplex $N(\f_3)$ with the coordinate order reversed.}
    \label{fig:orthoscheme}
\end{figure}

\begin{corollary}
\label[corollary]{cor:LB-f}
For every $n>0$,
there is 
$$F \in \relu_{n,2} \cap \relu^+$$
so that
$$F \not \in \relu^+_{n,n-1}.$$
\end{corollary}

The complete proof of \Cref{thm:LB-f} is given in Appendix~\ref{sec:Monotone}.

Iterated composition has been linked before to understanding the role of depth in the expressivity of neural networks~\citep{mhaskar2017and,telgarsky2016benefits}. The functions $\f_n$ appear to be new to this study and may prove useful for further results regarding the connections between expressivity and depth.

\section{Input convex neural networks}
As opposed to monotone $\relu$ networks, ICNNs can compute
any $\CPWL$ convex function~\citep{chen2018optimal}; see also~\citep{huang2020convex,gagneux2025convexity}.
In particular, ICNNs can compute
$\MAX_n$ for every $n$.
Our techniques allow to compute the $\ICNN$ depth complexity of $\MAX_n$.

\begin{theorem}
\label{thm:MaxICNN}
For every $n>1$,
$$\MAX_n \in \ICNN_{n,n-1}$$
but for every $k < n-1$,
$$\MAX_n \not \in \ICNN_{n,k}.$$
\end{theorem}

The value $n-1$ arises because the Newton polytope of $\MAX_n$ is an $(n-1)$-dimensional simplex.
If one considers $\max\{0,x_1,\ldots,x_n\}$ instead, whose Newton polytope is an $n$-dimensional simplex, the $\ICNN$ depth complexity is exactly $n$.

\textit{Proof Sketch.} The proof closely follows that of \cref{thm:LB-f}, using the $(n-1)$-dimensional simplex $N(\MAX_n)=\conv\{e_1,\ldots,e_n\}$ instead of the orthoscheme, since the $\ICNN$ model allows the more general operation of joining a general point and not simply the origin.

A related result was proved by~\citet{valerdi2024minimal}. He considered a polytope-construction model that corresponds to ICNN-like networks that use
$\MAX_2$ gates instead of $\relu$ gates.
His ideas imply that in this
$\MAX_2$-variant of ICNNs,
the depth complexity of $\MAX_n$
is $\Theta(\log n)$~\citep{valerdi2024minimal}.

There are a few works that proved $\ICNN$-depth lower bounds for some low-dimensional 
convex $\CPWL$ function.
\citet{gagneux2025convexity} showed that there is
a convex 
$$F \in \relu_{2,2}$$
so that 
$$F \not \in \ICNN_{2,2}.$$
\citet{valerdi2024minimal} constructed for every $k>0$, a function
$$F \in \ICNN_{4,d}$$ with $d \leq 2^{O(k)}$ so that 
$$F \not \in \ICNN_{4,k}.$$
This already implies a strong depth separation between general $\relu$ networks and $\ICNN$s in dimension four. By \citet{brunck2026two}, the general network can moreover be taken to have depth $2$.
Valerdi's construction is based on special cyclic polytopes which exist in dimension $n \geq 4$.
He left the problem of a similar construction in dimension $n=3$ open.
We solve this problem for the ICNN model (which is a weaker model than the model Valerdi considered). 
We obtain such a strong depth separation already in dimension three.

\begin{theorem}
\label{thm-icnn}
There is a constant $C>0$
so that the following holds.
For every $k > 0$,
there is 
$$F \in \ICNN_{3, d}\cap\relu_{3,2}$$
with $d \leq C k^2$
so that
$$F \not \in \ICNN_{3,k}.$$
\end{theorem}

The proof of this theorem is given in Appendix~\ref{section:ICNN-lower-bound}.

\begin{figure}
\centering
\includegraphics[page=1]{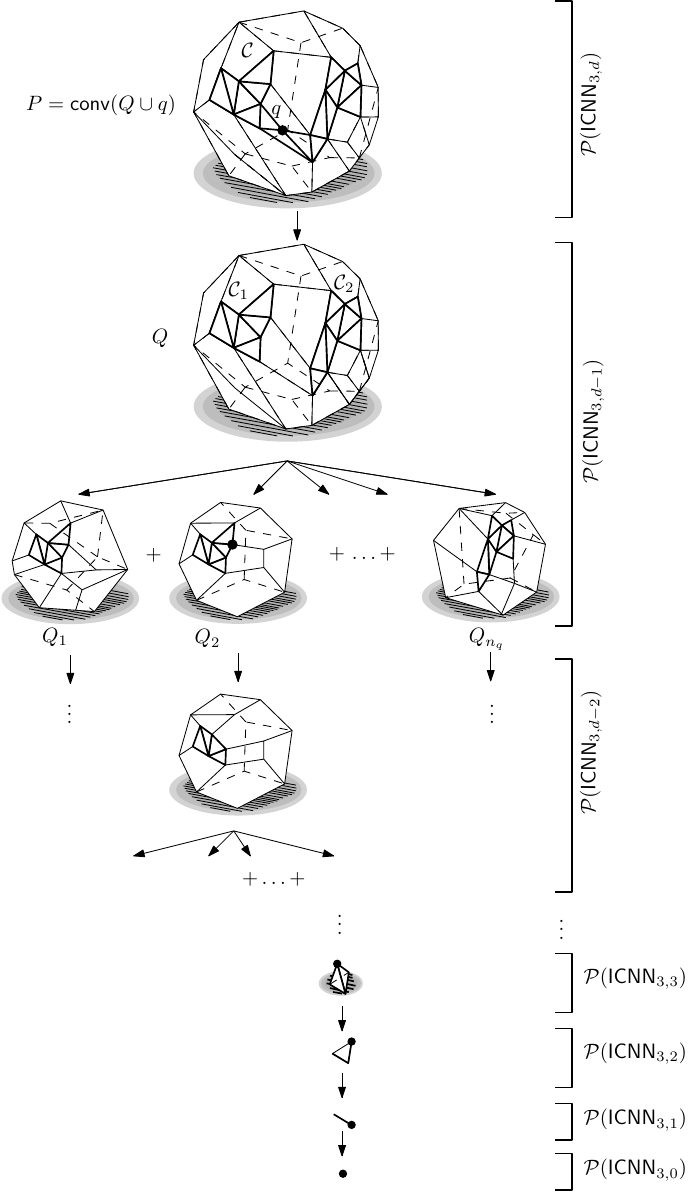}
\caption{. An example of a polytope tree.}
\label{fig:sum-chains}
\end{figure}

\textit{Proof Sketch.} The proof of this theorem draws a surprising connection with planar isoperimetric phenomena.
We pinpoint a complexity measure of a polytope $P$ which survives in at least one summand of any Minkowski decomposition: if $P$ can be written as a sum of polytopes, one of the summands must be at least as complex as $P$.
Geometrically, this measure keeps track of the biggest connected chain of triangles among the facets of $P$.
Working through the network from the output gate towards the input gates, we see that a $\relu$ operation can remove a vertex and thus can only affect triangles incident to that vertex.
The network must eventually get rid of this complexity and lower it to that of a point (see \Cref{fig:sum-chains}).
Any network representing $P$ must thus represent a strategy to delete vertices and separate the original chain of triangles into points.
The key insight is that the cost of an optimal deletion strategy remains subject to isoperimetry: there is an optimal tradeoff between the number of triangles (the area) and the number of vertices on their boundary (the perimeter) which follows a square root law.

\section{Limitations and future work}
\label{sec:limitations}

Our results focus on exact representability and depth lower bounds for two restricted classes of $\relu$ networks, namely monotone $\relu$ networks and $\ICNN$s. The broader exact-depth problem for unrestricted $\relu$ networks has recently been resolved: every $\CPWL$ function can be computed by a depth-two $\relu$ network~\citep{brunck2026two}.
We also note that our lower bounds are depth lower bounds with no restriction on network size beyond finiteness and it would be interesting to explore depth-width tradeoffs in future work. On the monotone side, \Cref{lem:isoto}, \Cref{prop:isoPlane} and \Cref{prop:example} show that monotonicity, convexity, and isotonicity of the subgradient capture part of the structure of $\relu^+$ networks, but a complete characterization of the functions exactly representable by $\relu^+_n$ remains open. 

\section{Funding Transparency}
\label{sec:funding}

Egor Bakaev, Florestan Brunck and Amir Yehudayoff were funded by the Danish National Research Foundation (DNRF).

\bibliographystyle{plainnat}
\IfFileExists{mr.bib}{\bibliography{mr}}{\bibliography{../mr}}

\appendix

\section{Newton polytopes}
\label[appendix]{sec:NewP}
There is a deep connection between convex CPWL functions
and their Newton polytopes (see~\citep{zhang2018tropical,maclagan2021introduction,maragos2021tropical,hertrich2021towards}
and the references within).

We focus our attention on 
the set of homogeneous functions $\mathsf{HOM}$.
The reason is that the following geometric discussion is cleaner for this class of functions (a similar theory holds for the general non-homogeneous case).
Let $F \in \CPWL_n \cap \mathsf{HOM}$ be convex. The function $F$ is of the form $$F(x) = \max \{L_1(x),\ldots,L_m(x)\}$$
where $L_1,\ldots,L_m$ are linear functions $L_i(x) = \ip{v_i}{x}$, for some $v_i\in \R^n$.
The Newton polytope of $F$ is defined to be
$$N(F) = \conv( \{v_1,\ldots,v_m\})$$
where $\conv$ is the convex hull in $\R^n$ (see an example in \Cref{fig:newton}).
The function $F$ can be written as
$$F(x) = \max \{ \ip{x}{p}: p \in N(F) \};$$
it is sometimes called the support function
of the polytope $N(F)$.
It satisfies the following clean properties:\ for such functions $F_1,F_2$ and $a_1,a_2>0$,
$$N(a_1 F_1+a_2F_2) = a_1N(F_1)+a_2 N(F_2)$$
and
$$N(\relu(F_1)) = \conv (\{0\} \cup N(F_1)).$$

\begin{figure}\centering
\includegraphics[page=8]{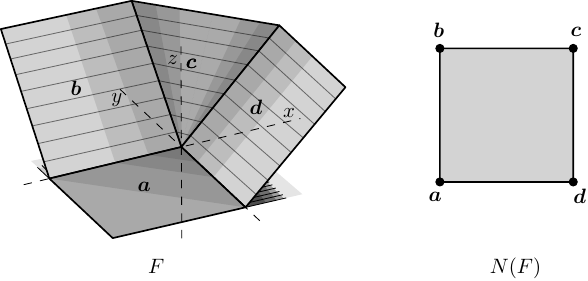}
\caption{. The graph of the function $F(x,y) = \max \{x,y,x+y,0\}$ and its Newton polytope $N(F)$ obtained as the convex hull of the four points $(0,0), (0,1), (1,0), (1,1)$. The sub-gradient $\partial F(0)$ is equal to the entire set $N(F)$.
}
\label{fig:newton}
\end{figure}

These properties translate monotone $\relu$ networks and ICNNs to convex geometry.
Instead of computing functions,
we build polytopes.
The two operations are  Minkowski sum
and ``adding zero''.
In the first layer of the computation,
we can add points in $\R^n$
(for ICNNs) and points in $\R_+^n$
(for monotone $\relu$ networks).
If we already constructed
convex polytopes $P_1,P_2,\ldots$
then with a sum operation we can construct
$$\sum_j a_j P_j$$
where $a_j > 0$.
With an ``add zero'' operation,
from a polytope $Q$, we can construct
the polytope
$$\conv(\{0\} \cup Q).$$
We get two circuit models for constructing polytopes
(which are weaker than the model considered by ~\citet{valerdi2024minimal}).
We get the following two families of polytopes
\begin{align*}
\cP ( \relu^+_{n,k})
= \{N(F) : F \in \relu^+_{n,k} \cap \mathsf{HOM} \} , \\
\cP ( \ICNN_{n,k})
= \{N(F) : F \in \ICNN_{n,k}  \cap \mathsf{HOM} \} .
\end{align*}
We also get a correspondence between the space of functions
$\relu^+_{n,k}$ and the space of polytopes $\cP(\relu^+_{n,k})$,
and between the space of functions
$\ICNN_{n,k}$ and the space of polytopes $\cP(\ICNN_{n,k})$.
Networks for functions give network for polytopes and vice versa. 

In the polytope setting, the difference between $\ICNN$s and $\relu^+$ networks 
is that the ``input points'' are from $\R^n$ and $\R^n_+$. Another difference is that in the $\ICNN$ model
the ``add zero'' operation
can be extended without increase in depth
to an ``add $q$'' operation for arbitrary points $q\in\R^n$. 
This can be seen via
$$\conv (P \cup \{q\}) = \conv ((P+\{-q\})\cup \{0\}) + \{q\}.$$
An ``add $q$'' operation
can be simulated by three operations with no increase in depth.
This observation immediately shows that
any polytope $P \subset \R^n$ with $m$ vertices belongs to $\cP(\ICNN_{n,m})$.
In other words, $\ICNN$s can compute any convex
$\CPWL$ function.

We demonstrate the power of this language by proving \Cref{prop:isoPlane} and
\Cref{prop:example}. 
A polytope $P \subseteq \R^n$ is the convex hull of finitely many points.
For a non-zero $u \in \R^n$ and a polytope $P \subset \R^n$,
denote by $H(P,u)$ the supporting hyperplane of $P$ in direction $u$,
and by $h(P,u)$ the support function:
$$h(P,u) = \max \{ \ip{x}{u} : x \in P\}.$$
That is, the normal to the hyperplane $H(P,u)$ is $u$ and it holds that $P \subset \{x \in \R^n : \ip{x}{u} \leq h(P,u)\}$
and that $P \cap H(P,u) \neq \emptyset$.
A set of the form $P \cap H(P,u)$ is called a face of $P$. 
The following is a standard; see e.g.~\citep{shephard1963decomposable}. We denote by $\S^{n-1}$ the standard unit sphere in $\R^n$.

\begin{fact}
\label[fact]{clm:supp}
For every $u \in \S^{n-1}$ and every polytopes $P,Q \subset \R^n$,
$$(H(P,u) \cap P) + (H(Q,u) \cap Q) = H(P+Q,u) \cap (P+Q).$$
\end{fact}

We shall use the following well-known properties of sub-gradients.

\begin{claim}
\label[claim]{clm:sub-grad}
Let $F_1,F_2:\mathbb R^n\rightarrow \mathbb \R$ be convex functions,
let $a_1,a_2 \geq 0$,  
and let $x\in\mathbb \R^n$.
The following properties hold:
\label[claim]{clm:subgradient}
\noindent\begin{enumerate}
\item 
${\partial (a_1 F_1+a_2 F_2)(x) = a_1 \partial F_1(x) + a_2 \partial F_2(x)}$.
\item If $F_1(x)=F_2(x)$
then 
$$\partial (\max \{F_1,F_2\})(x)=\conv (\partial F_1(x) \cup \partial F_2(x)).$$
\item If $F_1(x)>F_2(x)$
then 
$$\partial (\max \{F_1,F_2\})(x)= \partial F_1(x).$$

\end{enumerate}
\end{claim}

By \Cref{clm:sub-grad},
if $P \subseteq \R^n$ is a convex polytope with vertex-set $V$ and $$F(x) = \max \{ \ip{x}{p} : p \in P\} = \max \{\ip{x}{v} : v \in V\},$$
then for all $x \neq 0$,
the sub-gradient $\partial F(x)$
is the face of $P$
of the form $\partial F(x) = P \cap H(P,x)$;
see \Cref{fig:subgradient}.
The sub-gradient at zero is
$\partial F(0) = P$.

\begin{figure}\centering
\includegraphics[page=16]{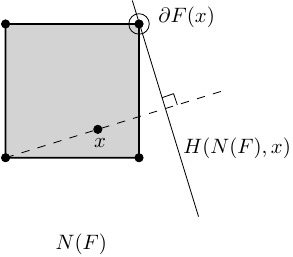}
\caption{. The sub-gradient at $x$ of the function $F=\max\{0,x_1,x_2,x_1+x_2\}$.
}
\label{fig:subgradient}
\end{figure}

We describe the isotonic-gradient property using the following language.
We say that a convex polytope $P \subset \R^n$
has \emph{positive edges} if there is a non-negative orientation of its edges;\footnote{If a polytope is a point, then it has positive edges.}
that is, if $e$ is the edge of $P$ between vertices $u$ and $v$
then either $u-v$ or $v-u$ is in $\R_+^n$. 
If there is such an orientation, then it is unique (at most one of $u-v$ and $v-u$ can be non-negative).

The following claim gives us a clean way to verify that the sub-gradient of a function is isotonic. 

\begin{claim}
\label[claim]{clm:posEiffiso}
Let $P \subseteq \R_+^n$ be a convex polytope, and let $F(x) = \max \{ \ip{x}{p} : p \in P\}$.
Then, the following conditions are equivalent:

\begin{enumerate}[(i)]
    \item $P$ has positive edges.
    \item The subgradient of $F$ is isotonic.
\end{enumerate}

\end{claim}

\begin{proof}

\noindent\emph{(i) implies (ii)}.
Assume that $P$ has positive edges.
Let $V$ be the vertices of~$P$
so that
$F(x) = \max \{\ip{x}{v} : v \in V\}$.
Let $x \leq y$.
Our goal is to prove
that $\partial F(x) \leq \partial F(y)$. 
When we continuously move $z$
on the line segment from $x$ to $y$,
the sets $\partial F(z)$
form a connected sequence of faces of $P$. 
By the transitivity of $\leq$,
it suffices to consider two consecutive faces in this sequence. 
There are two cases to consider.
Start by considering $z \in \R^n$ and 
$u = \eps (y-x) \in \R_+^n$ be of small norm so that
 $E_+ := \partial F(z+u)$ and $E := \partial F(z)$ 
are two consecutive faces
and $E_+$ 
is a face of $E$.
By assumption, all edges in $E$
can be directed to be non-negative.
The $1$-skeleton 
of $E$ is therefore
a directed acyclic graph (DAG).
There is a sink $p$ in the graph.
For all edges $q \to p$ in $E$, because $z$ is normal to $E$,
$$0 \leq \ip{p-q}{u}
= \ip{p-q}{z+u}$$
so 
$$\ip{q}{z+u} \leq \ip{p}{z+u}.$$
This means that
$p$ is a local and hence, by convexity of $P$,
also a global maximum of $t \mapsto \ip{t}{z+u}$,
which implies that 
$p \in E_+$.
It follows that for every vertex $v$ in $E$,
there is a sink above it
in $E_+$. This can be extended via convex combinations
to all of $E$
so that
$$E \leq E_+.$$
In the second case, $E$ is a face of $E_+$ and we can use a similar argument where ``sink'' is replaced by ``source''.

\medskip

\noindent\emph{(ii) implies (i)}. Assume that the gradient of $F$ is isotonic. 
Let $e=[p,q]$ be an edge of $P$.
Assume towards a contradiction that some of the entries of $p-q$
are negative and some are positive. 
It follows that there is $v' \in \R^n$ with positive entries so that $\ip{v'}{p-q}=0$. 
Let $v$ have positive entries be of the form $v=p-q+\alpha v'$ for $\alpha > 0$.
It follows that $\ip{v}{p-q}>0$.
Let $z$ be so that $\partial F(z)$ is the edge $e$.
Let $x = z-\delta v$ and $y = z+\delta v$ for small enough $\delta > 0$
so that $\partial F(x) = \{q\}$ and
$\partial F(y) = \{p\}$.
We get a contradiction;
although $x \leq y$,
the sub-gradients at $x,y$
are incomparable. 

\end{proof}

A central idea in our lower bounds
is the notion of indecomposable
polytopes.
The polytopes $P ,Q\subset \R^n$ are homothetic if there are $a \geq 0$ and $b \in \R^n$ so that $P=aQ+b$.
A polytope $P$ is called indecomposable if
for all $P_1,P_2,\ldots,P_m$ so that $P = \sum_j P_j$, each $P_j$ is homothetic to $P$. Simplices are a central example of indecomposable polytopes. 

\begin{fact}[e.g.~\citep{shephard1963decomposable,gr2003unbaum}]
\label[fact]{fact:triangle}
Simplices are indecomposable.
\end{fact}

\begin{proof}[Proof of \Cref{prop:example}]
Set $P_* = \conv(V)$ with
$$V = \{(0,0,0),(1,0,0),(0,1,0),(1,1,0),(1,1,1)\}.$$
The polytope $P_*$ is a pyramid with a square base (see \Cref{fig:pyramid}).
First, let us explain why $F$, the function that corresponds to $P_*$,
is homogeneous, monotone, convex and has isotonic gradients. 
It is monotone because $P_* \subset \R_+^3$.
It is homogeneous
and convex as the maximum of linear functions.
It has isotonic gradients because $P_*$ has positive edges, by \Cref{clm:posEiffiso}.

Proving that $P_*$ is not in $\cP(\relu^+_3)$ is based on 
the fact that $P_*$ is indecomposable
(see \citep[Theorem 12]{shephard1963decomposable}).
This implies that if
$P_*$ is the output of a Minkowski
sum gate, then a positive scaling of $P_*$ is also an output of a previous gate.
So, Minkowski sum gates are useless
for generating $P_*$.

Next, consider ``add zero'' gates. 
The claim is that if $P_* = \conv( \{0\} \cup Q)$ with $Q \in \cP(\relu^+_{3})$ then $Q=P_*$.
Indeed, 
if $P_* = \conv( \{0\} \cup Q)$ then
$V \setminus \{0\} \in Q$.
Denote by $E$ the $\{e_1,e_2\}$-plane,
and consider the two-dimensional polytope
$Q' = E \cap Q$. 
The sequence of $\cP(\relu^+_{3})$ gates that generate $P_*$ lead to a sequence of
$\cP(\relu^+_{2})$ that generate $Q'$.
This can be done by replacing
the ``point gates'' as follows,
and keeping all other ``inner gates'' as is.
If the point $p = (p_1,p_2,p_3) \in \R_+^3$ appears
in the generation of $Q$,
then if $p \in E$ replace $p$ by $(p_1,p_2) \in \R_+^2$ 
and if $p \not \in E$ then delete $p$.
It follows by induction 
that if $P$ is computed by some gate for $Q$,
then the corresponding gate for $Q'$ computes
$E \cap P$.

\begin{figure}\centering
\includegraphics[page=10]{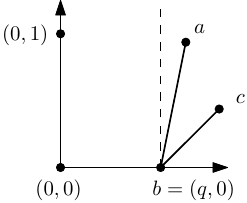}
\caption{. The structure of $Q'$.}
\label{fig:monotone01and10}
\end{figure}

The polytope $Q'$ contains $(0,1)$ and $(1,0)$.
We prove that the only way to do that in 
$\cP(\relu^+_2)$ is to also have $(0,0)$ inside $Q'$.
This completes the proof, because then $0 \in Q$ and so $Q = P_*$.
Select a vertex $b:=(q,0)$ of $E \cap Q$ for minimal $q$.
Because $(1,0) \in Q'$, we know $0\leq q \leq 1$.
Let $a,c$ be the vertices of $Q'$ adjacent to $b$; see \Cref{fig:monotone01and10}.
Because $Q' \subset \R_+^2$ has positive edges, 
we know that $a-b$ is in $\R_+^2$.
Similarly, $c-b$ is in $\R_+^2$. Because $Q'$ is convex, for every point $t \in Q'$, we know that $t-b$ is in $\R_+^2$.
In particular, for $t=(0,1)$, we have $t-b = (-q,1)$ is in $\R_+^2$. So, $q = 0$ and
$b=(0,0) \in Q'$.

\end{proof}

For the proof of \Cref{prop:isoPlane},
we use the following lemma which can be found (without a proof) in \citep[Chapter 15.1, Exercise 2]{gr2003unbaum}
and \citep[Exercise 4-12]{yaglom1961convex}.

\begin{lemma}[]
\label[lemma]{lem:planeMinsum}
Every polygon is a Minkowski sum of segments and triangles.
\end{lemma}

\begin{figure}\centering
\includegraphics[page=11]{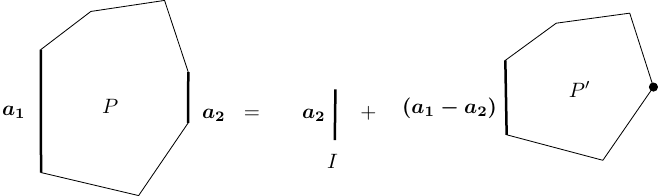}
\caption{. Representing polygon $P$ as a Minkowski sum of segment $I$ and polygon $P'$. Illustration for the case when $P$ has two parallel sides with lengths $a_1$ and $a_2$; $a_1>a_2$.}
\label{fig:decomposition-segment}
\end{figure}

\begin{figure}\centering
\includegraphics[page=12]{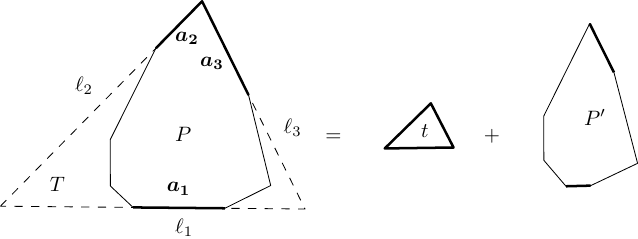}
\caption{. Representing polygon $P$ as a Minkowski sum of triangle $t$ and polygon $P'$. Polygon $P$ lies inside triangle $T$ that is homothetic to~$t$.
}
\label{fig:decomposition-triangle}
\end{figure}

To prove the lemma, we rely on the following simple claim.

\begin{claim}
\label[claim]{clm:P-Q=tran}
If $P$ and $Q$ are convex polygons in $\R^2$, and for all $u \in \S^{1}$, 
the two faces 
$H(P,u) \cap P$ and $H(Q,u) \cap Q$ differ only by translation,
then $P$ and $Q$ differ only by translation.
\end{claim}

\begin{proof}[Proof of \Cref{lem:planeMinsum}]
The proof is by induction on number of vertices
$m$ of the polytope $P$.
For $m \leq 3$ the statement is trivial, so now assume that $m \geq 4$.
There are a few cases to consider.

\medskip

\noindent
{\bf Case (a).} Suppose that $P$ has two parallel sides, with lengths $a_1$ and $a_2$;
see \Cref{fig:decomposition-segment}.
Shorten these sides by $a:=\min\{a_1,a_2\}$
to get a polytope $P'$ with fewer edges.
One of the two edges of $P$ became smaller and at least one of the edges vanished to a point. 
We can write $P$ as $P'$ plus
an interval $I$ of length $a$ that is parallel
to the edges that were contracted. 
Indeed, by \Cref{clm:supp}, for all $u \in \S^{1}$,
$$(H(P',u) \cap P') + (H(I,u) \cap I) = H(P'+I,u) \cap (P'+I).$$
If $u \in \S^1$ is orthogonal to $I$,
then $H(I,u) \cap I = I$
and 
$(H(P',u) \cap P') + I$ is a translate of $H(P,u) \cap P$.
If $u \in \S^1$ is not orthogonal to $I$,
then $H(I,u) \cap I$ is a point
and 
$H(P',u) \cap P'$ 
is a translate of $H(P,u) \cap P$.
By \Cref{clm:P-Q=tran},
we see that $P = P' + I$.

\medskip

\noindent
{\bf Case (b).} Suppose that $P$ has no parallel edges; see \Cref{fig:decomposition-triangle}. Let us select arbitrary edge $a_1$. 
Denote by $\ell_1$ the line that contains $a_1$.
Let $v$ be the vertex that is the farthest from the line $\ell_1$;
it is unique because there are no parallel edges. 
Denote by $a_2$ and $a_3$ the two edges that share the vertex $v$.
Denote by $\ell_2$ and $\ell_3$ the two lines to which the edges belong. The polytope $P$ has no parallel sides, so $\ell_1$, $\ell_2$, $\ell_3$ form a triangle $T$. 
The polytope $P$ is contained in the triangle $T$. 
Denote by $b_1,b_2,b_3$
the three edges of $T$
numbered so that $a_i$ is contained in $b_i$.
Denote by $|a|$ the length of the edge $a$.
Let $t = m T$ be a positive homothet of $T$ where
$$
m := \min\big\{\tfrac{|a_i|}{|b_i|} : i \in [3]\big\} > 0.
$$
Let $P'$ be the polytope obtained from $P$ by shortening the edge $a_i$
by $m b_i$.
The polytope $P'$ has fewer vertices than $P$.

We can write $P$ as $P' + t$. The proof of this is similar to case (a).
By \Cref{clm:supp}, for all $u \in \S^{1}$,
$$(H(P',u) \cap P') + (H(t,u) \cap t) = H(P'+t,u) \cap (P'+t).$$
For arbitrary $u \in \S^1$,
face $H(t,u) \cap t$ is a segment or a vertex. There are two cases.

If $|H(t,u) \cap t| = c > 0$ (so it is a segment of non-zero length)
then 
$c = m |b_i|$ for some $i \in [3]$
and $|H(P',u) \cap P'| + c$ is exactly $|H(P,u) \cap P|$ by construction of $P'$.
So, $(H(P',u) \cap P') + (H(t,u) \cap t)$ is a translate of $H(P,u) \cap P$.

If $|H(t,u) \cap t| = 0$ (so it is a point)
then $H(P',u) \cap P'$ is a translate of $H(P,u) \cap P$ by construction of $P'$.

By \Cref{clm:P-Q=tran},
we see that $P = P' + t$.
\end{proof}

\begin{proof}[Proof of \Cref{prop:isoPlane}]
We first use \Cref{clm:posEiffiso}
to translate ``isotonic gradients of $F$''
to ``positive edges of $P = N(F)$''. 
\Cref{lem:planeMinsum}
says that we can write $P$ as the Minkowski sum of segments and triangles. 
\Cref{clm:supp}
shows that these segments are positive and the edges of the triangles are positive.
It is easy to verify the proposition
for segments and triangles. For example, consider the triangle with vertices $v_1,v_2,v_3 \in \R_+^2$ so that $v_1 \leq v_2 \leq v_3$. We can first generate $e = v_2-v_1 + \conv(\{0\} \cup \{v_3-v_2 \})$ and then generate $v_1 + \conv(\{0\} \cup e)$. \end{proof}

\section{Lower bounds for \texorpdfstring{$\ICNN$s}{ICNNs}}

\label[appendix]{section:ICNN-lower-bound}
This section is dedicated to proving \Cref{thm-icnn}. 
We shall in fact prove the following more general statement (a polytope with $m$ vertices can be generated in depth $m$).

\begin{theorem}
\label{thm-icnn-3D}
There is a constant $C>0$ such that the following holds.
For every $m\geq 4$, there exists a 3-dimensional polytope $P$ with at most $m$ vertices so that for all $k \leq C \sqrt{m}$, $$P\notin \cP ( \ICNN_{3,k}).$$
\end{theorem}

Recall that for $u \in \S^{n-1}$ and a polytope $P \subset \R^n$, we denote by $H(P,u)$ the supporting hyperplane of $P$ in direction $u$.
Additionally, we denote by $P_u$ the face of $P$ supported by $H(P,u)$.

\subsection*{Chains of triangles}
A collection $\mathcal C$ of triangles in $\R^n$ is called a \emph{chain of triangles} if it is a pseudomanifold;
that is, (1) for every two triangles $t,t'$ in $\mathcal C$,
there is a sequence $t_0,t_1,\ldots,t_m$
of triangles in $\mathcal C$ so that
$t_0 = t$, $t_m = t'$
and for every $i \in [m]$,
the two triangles $t_{i-1},t_i$
share an edge, and (2) every edge belongs to at most two triangles in $\mathcal{C}$; see illustration in
\Cref{fig:chain}.

The collection of triangles $\mathcal K$ is a homothet of 
the collection of triangles $\mathcal C$ if $\mathcal K = a \cdot {\mathcal C}+b$
for $a \geq 0$ and $b \in \R^n$.
We sometimes call $a$ the dilation factor. If $a>0$, we say that $\mathcal{K}$ is a positive homothet of $\mathcal{C}$. We say that a polytope $P$ contains the collection of triangles $\mathcal C$ if there is a positive homothet $\mathcal{K}$ of $\mathcal{C}$ so that each triangle $T$ in $\mathcal{K}$ is a face of $P$.

We are ready for our main definition.
A collection $\mathcal C$ of triangles in $\R^n$ is called indecomposable
if for every polytope $P \subset \R^n$
that contains $\mathcal C$,
and for every polytopes $P_1,\ldots,P_m$ so that $P = \sum_j P_j$, for all $j \in [m]$, the polytope $P_j$ contains a homothetic copy of $\mathcal C$.

\begin{figure}\centering
\includegraphics[page=4]{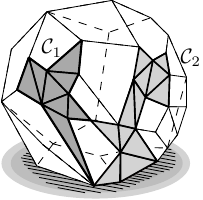}
\caption{. Two maximal chains of triangles $\mathcal C_1$ and $\mathcal C_2$. Even though $\mathcal C_1$ and $\mathcal C_2$ share two vertices, the set $\mathcal C_1 \cup \mathcal C_2$ is \emph{not} a chain of triangles.}
\label{fig:chain}
\end{figure}

Following ideas of Shephard (see proof of (12) in~\citep{shephard1963decomposable}), we get the following important lemma.
A similar property is also central in Valerdi's work~\citeyearpar{valerdi2024minimal}.

\begin{lemma}
\label[lemma]{lemma:chains}
All chains of triangles are indecomposable.
\end{lemma}

\begin{proof}[Proof of \Cref{lemma:chains}]
Let $T$ be a triangle in $\mathcal{C}$.
Let $P = Q + Q'$ be a polytope containing $\mathcal C$. 
Let $u \in \S^{n-1}$ be so that $T = P_u$.
\Cref{clm:supp} and \Cref{fact:triangle} tell us that $Q_u$ is a (translate of a) triangle of the form $\lambda_T \cdot T$ for some $\lambda_T \geq 0$. Similarly,
every edge $e$ in $T$ has a dilation factor $\lambda_e$ in $Q$.
All three edges $e$ of $T$ appear in $Q$ with the same dilation factor $\lambda_e = \lambda_T$; see e.g.\ the proof of (12) in~\citep{shephard1963decomposable}.

Because the chain of triangles $\mathcal{C}$ is a pseudomanifold, 
all edges in the triangle chain $\mathcal{C}$ have the same dilation factor $\lambda_*$. 

It follows that if $p,p'$ are two vertices of $P$ and belong to $\mathcal{C}$ and if $q,q'$ are the two vertices of $Q$ that correspond to $p,p'$ then $q'-q = \lambda_*(p'-p)$.

It follows that a translation of $\lambda_* \mathcal{C}$ is in $Q$.

\end{proof}

\subsection*{Proof outline}
Before we provide the full proof,
which is rather technical,
we provide a high-level description.

To prove the lower bound, we identify a property of polytopes that make them ``complex''. In a nutshell, a polytope $P$ is complex if it contains a ``well connected'' chain of triangles.

We keep track of the evolution of the chain of triangles from the output gate of the network towards the input gates; see illustration in \Cref{fig:sum-chains}.
Let $P$ be some polytope with a given chain of triangles that is computed by a gate in the network. 
If $P$ is obtained as $P=\sum_j P_j$ then a positive homothet of its triangle chain is present in one of the $P_j$.
In other words, one of the $P_j$'s is as complex as $P$.
If $P = \conv(\{0\} \cup Q)$, then 
the triangle chain 
of $Q$ could be different from that of $P$, but only in one vertex. 
Again, if $P$ is complex then $Q$ should be at least somewhat complex. If $P$ is computed in an input gate then it is a point with no chain of triangles.

The lower bound is proved for a polytope $P:=P_r$ that contains a ``very well connected'' chain of triangles (it is defined below). As explained above, the network for $P$ must obliterate its  chain of triangles. We prove that this must require many ``add $q$'' operations. 
For simplicity, we focus on the ``combinatorial data'' of the chain of triangles that we encode by a graph.

This leads to the following type of ``game''.\footnote{This is a single player game (a ``puzzle'').}
The game is played over a graph~$G$. The goal is to shatter the graph; break it down into single vertices. 
Deleting the neighborhood of a selected vertex has a unit cost.
If this deletion segmented the graph to a few connected components, the costs of the components are not summed;
the cost is the maximum over the cost of the components. 
A strategy in the game corresponds to a tree of deletion moves on the vertices of the graph.
The nodes in the tree are selected vertices
and the branchings in the tree correspond to different connected components. 
The goal of the game is to shatter the graph with the minimum cost possible. 
We are interested in the cost of an optimal strategy.

The answer turns out to (mainly) depend
on the isoperimetric properties of the graph $G$. In the planar graph we use, every set of $\ell$ vertices has a boundary of size~$\Omega(\sqrt{\ell})$, which is optimal for planar graphs. 
This eventually leads to an $\Omega(\sqrt{m})$ lower bound on the cost
of the game and consequently on the depth of the network.

\subsection*{The construction of the polytope}
We begin by building our 3-dimensional polytope $P:=P_r$. 
Consider the 2-dimensional infinite triangular lattice, where six equilateral triangles meet at each vertex.
This lattice is a planar embedding of an infinite graph $G$.
Choose a vertex $o$ of $G$, which we think of as being the ``origin''. 
For any $r\in \N$, write $B_r$ to denote the ball of radius $r$ around $o$; i.e., the subgraph of $G$ induced by the vertices at graph-distance at most $r$ from $o$, see \Cref{fig:G_k}.

\begin{figure}\centering
\includegraphics[page=2]{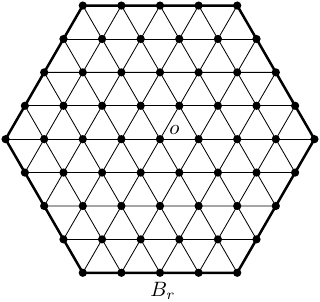}
\caption{. The graph $B_r$.}
\label{fig:G_k}
\end{figure}

We are going to use Steinitz' theorem (see e.g.~\citep[Section 13.1]{gr2003unbaum}).
A graph $G$ is \emph{$3$-connected} if removing any $2$ vertices from $G$ keeps $G$ connected.
Steinitz' theorem states that a planar graph $G$ corresponds to the vertices and edges of a $3$-dimensional polytope $P$ 
if and only if $G$ is $3$-connected.
The following claim follows by induction on $r$.

\begin{claim}
    \label[claim]{clm:3-connected}
    $B_r$ is planar and $3$-connected. 
\end{claim}

Set $P = P_r$ to be the polytope given for $B_r$ by Steinitz' theorem.

\begin{remark}
We present an explicit construction of the polytope $P_r$ in \Cref{fig:explicit-construction}. This natural construction is based on the inverse stereographic projection of the embedding of $B_r$ in a plane $E$ onto a sphere $S$ tangent to $E$ at the origin $o$.

\end{remark}

\begin{figure}\centering
\includegraphics[page=3]{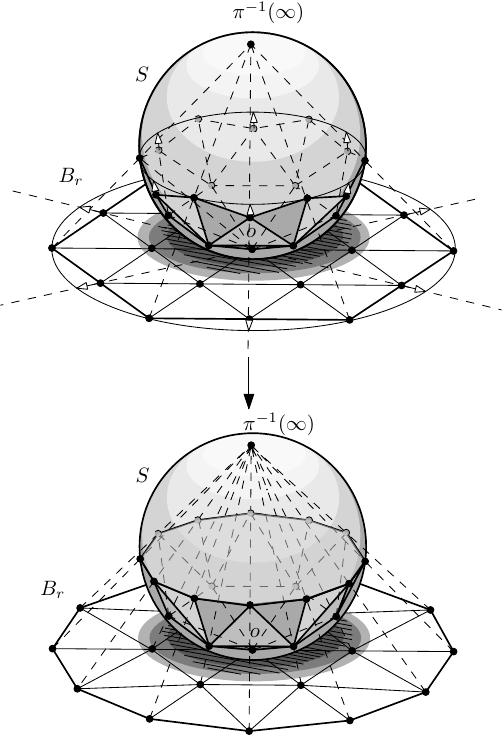}
\caption{. Building the polytope with a projection.}
\label{fig:explicit-construction}
\end{figure}

\subsection*{A coloring game}
As explained above, to prove the lower bound we can ignore some of the information about the polytopes computed by the network.
\Cref{lemma:chains} tells us that it is a good idea to focus on their chains of triangles. 
We can think of chains of triangles as graphs. 
Instead of a tree of polytopes (see the  illustration in \Cref{fig:sum-chains}), we consider a tree of graphs.
The root of the tree corresponds to the full set of vertices $V(B_r)$.
The leaves of the tree correspond to single vertices in $V(B_r)$.

\begin{algorithm}
	\caption*{$\boldsymbol{ \operatorname{color}}\pmb(V_B\pmb)$} 
	\begin{algorithmic}
\If{$|V_B| \leq 1$}
\State return $|V_B|$
\Else   
\State select $q\in B_1(V_B)$ \textcolor{gray}{// black or white}
            \State set $V_B\coloneqq V_B\setminus B_1(q)$ \textcolor{gray}{// color $B_1(q)$ white}
            \State denote the (new) connected components of $V_B$             by $C_q^1, C_q^2,\ldots, C_q^\ell$
                \State recursively compute $c_1 = \boldsymbol{ \operatorname{color}}\pmb( C_q^1\pmb)$
                \State recursively compute  $c_2 = \boldsymbol{ \operatorname{color}}\pmb(C_q^2\pmb)$
                \State $\vdots$
                \State recursively compute  $c_\ell = \boldsymbol{ \operatorname{color}}\pmb(C_q^\ell \pmb)$
             \EndIf
             \State return $c = 1+ \max (\{0\} \cup \{c_i : i \in [\ell]\})$
	\end{algorithmic} 
\end{algorithm}

For $V_B \subseteq V(B_r)$, the game $\pmb{ \operatorname{color}(V_B)}$ is defined above.
The vertices in $V_B$ are called black vertices and the vertices not in $V_B$ are called white vertices. 
The white vertices are thought of as deleted from the graph. 

\begin{figure}\centering
\includegraphics[page=5]{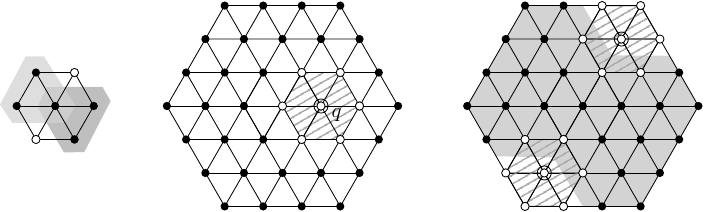}
\caption{. Balls in $B_r$.}
\label{fig:deletion-connected}
\end{figure}

The connected components of $V_B$ are the connected components of the graph induced by $V_B.$
At each step, a vertex $q\in B_1(V_B)$ is chosen and its neighborhood is colored white; see \Cref{fig:deletion-connected}.
The graph is then broken into the new black connected components. 
The game continues in each component separately. 
The goal is to color all vertices white.

A strategy for the game selects the next vertex whose neighborhood is colored white.
For each strategy, the game $\pmb{ \operatorname{color}(V(B_r))}$ returns a number that we think of as the full cost of the strategy. The number
$\pmb{ \operatorname{color}(V_B)}$
is the cost when starting at $V_B$.

Each strategy for $\pmb{ \operatorname{color}(V(B_r))}$ leads to a rooted directed tree; see \Cref{fig:algorithm}.
The root corresponds to $V(B_r)$.
Each non-leaf node in the tree corresponds to some $V_B$ and the selection of $q$ for $V_B$.
From a node $V_B$, there are edges 
to the nodes $C^1_q,\ldots,C^\ell_q$.
A set $V_B$ is obtained during the execution of $\pmb{ \operatorname{color}(V(B_r))}$ with some strategy if it corresponds to some node in the tree. 
There is a unique path in the tree from the root to the node $V_B$.
We associate a set of vertices $L(V_B)$ and a set of triangles $\mathcal{T}(V_B)$ to $V_B$.
The set of vertices ($q$'s) selected by the strategy on the path from the root to~$V_B$ is denoted by $L(V_B)$.

For a subset $U$ of the vertices of $V(B_r)$, denote by $B_1(U)$ the set of all vertices in $B_r$ of distance at most one from $U$,
and denote by $\mathcal{T}(U)$
the set of all triangles in $B_r$
that are contained in the graph induced by $B_1(U)$.
We think of $\mathcal{T}(U)$ both as a collection of triangles in the graph $B_r$
and as a collection of triangles in $3$-dimensional space. Each triangle in $B_r$ is embedded in $\R^3$ via the polytope $P_r$.
Every collection of triangles $\mathcal{T}(U)$ in the graph is also embedded in $\R^3$ via $P_r$. We shall think about $\mathcal{T}(U)$ also as this subset of $\R^3$ and we shall say that a polytope $Q$ contains $\mathcal{T}(U)$ if its boundary contains a positive homothet of $\mathcal{T}(U)$.

\begin{claim}
\label[claim]{clm:Cischain}

For every $U \subseteq V(B_r)$,
if the graph induced by $U$ is connected then 
$\mathcal{T}(U)$ is a chain 
of triangles.
\end{claim}

\begin{proof}
For every vertex $u \in V(B_r)$,
it holds that $\mathcal{T}(\{u\})$ is a pseudomanifold. The claim follows because the graph induced by $U$ is connected. 
\end{proof}

Next, we explain how $\ICNN$s lead
to strategies.

\begin{lemma}
    \label[lemma]{lemma:network-alg}
Let $V_B \subseteq V(B_r)$ be a set 
that is obtained during the execution of $\pmb{ \operatorname{color}(V(B_r))}$ with some strategy. 
If $Q \in \cP(\ICNN_{3,k})$ is a polytope that contains $\mathcal C = \mathcal{T}(V_B)$ then there is a strategy so that
$${\boldsymbol{\operatorname{color}\pmb(V_B \pmb)}} \leq k.$$
\end{lemma}

\begin{proof}
The strategy is built by starting at the output gate of the network for $Q$ and going down the network; see \Cref{fig:sum-chains}.
(1) If $k=0$, then $Q$ is a point and $\mathcal{C}$ is empty. 
(2) If the output gate is a Minkowski sum gate, then by \Cref{clm:Cischain} and \Cref{lemma:chains} we know that one of the summands $Q_j$ contains $\mathcal C$ and we go down to this gate and apply induction (without making any selection in the strategy). (3) If the output gate is $\conv(\{q\} \cup Q')$, then there are two cases.
(3a) If $q$ is not a vertex in $\mathcal C$, then again go down to $Q'$ and select nothing.
(3b) Otherwise, $q\in B_1(V_B)$, so the strategy selects $q$,
and applies induction with the network for $Q'$. 

It remains to justify the inductive step.
In case (2), we know we explained why $Q_j$ contains $\mathcal{C}$. In case (3a), the polytope $Q'$ contains $\mathcal{C}$ because the vertices of $\mathcal{C}$ are vertices of $Q'$.
In case (3b), let $C_q^1,\ldots,C_q^\ell$ be the connected components of $V_B \setminus B_1(q)$.
By \Cref{clm:Cischain}, each component $C_q^i$ defines a chain of triangles $\mathcal{C}^i$ which appears in $Q'$. 
\end{proof}

The proof of the depth lower bound thus reduces to proving the following lemma.

\begin{lemma}
    \label[lemma]{lemma:running-time}
    For any coloring strategy,    $${\boldsymbol{\operatorname{color}\pmb(V(B_r) \pmb)}} \geq C r$$
    for some universal constant $C>0$.
\end{lemma}

\begin{remark}
   The lower bound in \Cref{lemma:running-time}
    is tight, as the strategy illustrated on \Cref{fig:tight} achieves it.
    Roughly speaking, with a coloring cost of $O(r)$ we can break the graph into connected components, each of size at most half of the graph we started with and then recurse.
Induction shows that $\boldsymbol{\operatorname{color}\pmb(V(B_r) \pmb)}=O(r)$.

\end{remark}

\begin{figure}\centering
\includegraphics[page=13]{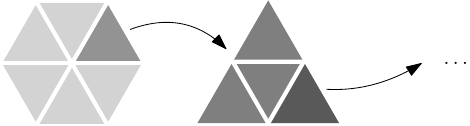}
\caption{. How the graph is broken down into connected components.}
\label{fig:tight}
\end{figure}

\begin{figure}\centering
\includegraphics[page=6]{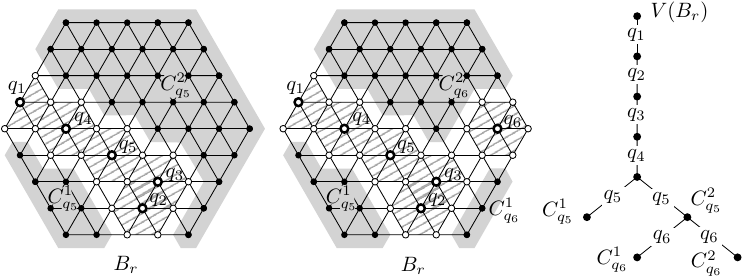}
\caption{. Each strategy for $\pmb{ \operatorname{color}(V(B_r))}$ leads to a rooted directed tree.}
\label{fig:algorithm}
\end{figure}

\subsection*{Isoperimetry}
The proof of the lower bound on the cost of the coloring game relies on isoperimetric properties of the graph
$B_r$. There are many known isoperimetric inequalities for similar scenarios (see e.g.~\citep{angel2018isoperimetric} and references therein), but we were unable to locate in the literature the particular one we need.
The (outer vertex) boundary $\partial K$ of a subset $K\subseteq V(B_r)$ is 
    \[
    \partial K := \{u\in V(B_r) \setminus K : \exists v \in K,\{v,u\} \in E(B_r) \} .
    \]
    
        \begin{lemma}
    \label[lemma]{lem:iso}
        There exists $C>0$ such that for all $K\subset V(B_r)$ so that \begin{align}
            \label{eqn:sizeK}
            \frac{1}{100} |V(B_r)| < |K|<\frac{99}{100} |V(B_r)|
        \end{align} 
        we have 
        $$|\partial K|>Cr.$$
    \end{lemma}

    \begin{proof}
 Let $K \subset V(B_r)$ be so that \eqref{eqn:sizeK} holds.  
The graph $B_r$ is embedded in the plane as part of the triangular grid.
Let $u \in \R^2$ be parallel to one of the edges of the triangles.

Partition $V(B_r)$ to fibers 
$\{V_i : i \in I \}$ according to lines that are parallel to $u$, where $I$ is of size $|I| = 2r+1$.

In other words, $V_i$ is the set of all vertices $v$ that belong to the same line (which is parallel to $u$), see \Cref{fig:fibers}.

We can think of $I$ as a set of points on the line~$u^\perp$. 
For $i \in I$, let $K_i = K \cap V_i$ be the fiber of $K$ over $i$. 
We call the fiber $K_i$ empty if $|K_i| = 0$.
We call the fiber $K_i$ full if $|K_i| = |V_i|$.
We call the fiber $K_i$ trivial if it is either empty or full.

Every non-trivial fiber contributes at least one to the boundary of $K$,
so if the number of non-trivial fibers is at least $\frac{r}{1000}$ then we are done. We can assume that the number of non-trivial fibers is less than $\frac{r}{1000}$. 
By~\eqref{eqn:sizeK},
there are full fibers and empty fibers.
Let $i_e,i_f$ be an empty fiber and a following full fiber.
There are $r$ vertex-disjoint paths between vertices of $i_e$ and $i_f$; see \Cref{fig:fibers}. 

Each of the paths has a pair of adjacent vertices, one in $K$ and one not in $K$. We can conclude that $|\partial K| \geq r$.
\end{proof}

     \begin{figure}\centering
     \includegraphics[page=14]{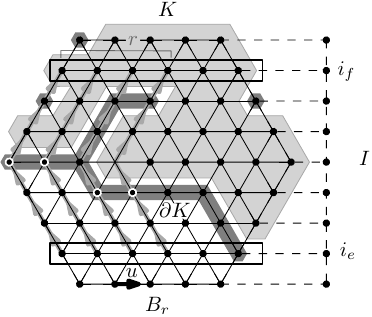}
 \caption{. Disjoint paths.}
    \label{fig:fibers}
     \end{figure}

\subsection*{The game tree}
    We begin with the following observation about the maximum degree of this recursion tree.
    
    \begin{claim}
    \label[claim]{claim:3-components}
        For all connected $V_B \subseteq V(B_r)$ and $q\in B_1(V_B)$, coloring $B_1(q)$ white creates at most six new black connected components.
        In other words,
        in $\pmb{ \operatorname{color}}$,
        we always have that $\ell \leq 6$.
    \end{claim}

     \begin{figure}\centering
     \includegraphics[page=7]{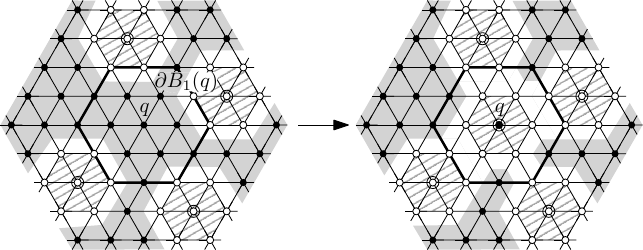}
 \caption{. For an ``inner $q$'' the set $U= \partial B_1(q)$ is a hexagon with twelve vertices.}
    \label{fig:connected-components}
     \end{figure}

    \begin{proof}
        Fix some vertex $q\in B_1(V_B)$ and assume that deleting $B_1(q)$ created $t$ new black connected components. Consider
        the set of vertices
        $U:=B_2(q)\setminus B_1(q) = \partial B_1(q)$.
        Each remaining component meets $U$, since $V_B$ is connected and $B_1(q)$ contains a black vertex.
        \Cref{fig:connected-components} 
        may be helpful here. 
        For an ``inner $q$'', this set is a hexagon with twelve vertices. 
        The size of $U$ is at most twelve. 
        The set $U$ must contain $t$ black vertices that are separated by $t$ white vertices. 
        This forces $t \leq 6$. 

    \end{proof}

    The next ingredient relates
    the game tree to boundaries of sets.

    \begin{claim}
    \label[claim]{claim:ancestors}
Let $V_B \subseteq V(B_r)$ be a set that is obtained during the execution of $\pmb{ \operatorname{color}(V(B_r))}$ with some strategy. 
Then, $\partial V_B \subseteq B_1(L)$
where $L  = L(V_B)$.
    \end{claim}
    \begin{proof}
    The proof is by induction starting at the root. In the induction base, $\partial V(B_r)$ is empty and there is nothing to prove. For the induction step, let $q$ be the vertex selected in ${\boldsymbol{\operatorname{color}\pmb(V_B \pmb)}}$. For each of the components $C_q^i$, the boundary $\partial C_q^i$ is contained in the union of $\partial V_B$ and $B_1(q)$.

    \end{proof}

\subsection*{The lower bound proof}

\medskip

    \begin{proof}[Proof of \Cref{lemma:running-time}] 
Define a sequence of sets $V_1,V_2,\ldots$ as follows. Start by $V_1 = V(B_r)$.
Given $V_i$, let $V_{i+1}$ be the largest connected component in the application of $\pmb{ \operatorname{color}(V_i)}$.
Each step removes at most seven vertices ($q$ and its at most six neighbors), so by \Cref{claim:3-components},
we know that $|V_{i+1}| \geq (|V_i|-7)/6$.
Let $V_*$ be the last set in the sequence so that $|V_i| \geq |V(B_r)|/12$.
It follows that $|V_{i+1}| < |V(B_r)|/12$
and so $|V_*| < (|V(B_r)|/2)+7$.
By \Cref{lem:iso}, we know that $|\partial V_*| > C r$.
Let $L_* \subseteq V(B_r)$ be the set of vertices chosen in the path leading to~$V_*$.
By \Cref{claim:ancestors},
we know that $|\partial V_*| \leq 7 |L_*|$.
Because $\pmb{ \operatorname{color}(V_1)}
\geq |L_*|$, we are done. 
    
\end{proof}

\begin{proof}[Proof of \Cref{thm-icnn-3D}]
The theorem is trivial for small values of $m$. 
Given $m$, let $r \approx \sqrt{m}$ be so that $\frac{m}{10} \leq |V(B_r)| \leq m-1$.
The polytope $P_r$ has at most $m$ vertices. 
Assume that $P_r \in \cP(\ICNN_{3,k})$.
By \Cref{lemma:network-alg}, 
there is a coloring strategy so that
$$k \geq {\boldsymbol{\operatorname{color}\pmb(V(B_r) \pmb)} }.$$
By \Cref{lemma:running-time},   $${\boldsymbol{\operatorname{color}\pmb(V(B_r) \pmb)}} \geq C r .$$
\end{proof}

\section{Inapproximability}
\label[appendix]{sec:Inapprox}

In this section, we prove that
$\MAX_2$ can not be approximated
by monotone $\relu$ networks. 
We start by proving that monotone
$\relu$ networks 
have isotonic and non-negative sub-gradients.

\begin{proof}[Proof of \Cref{lem:isoto}]
The proof is by induction.
The induction base 
corresponds to monotone affine functions for which the lemma holds.
For the induction step,
let $F = \relu(G)$ with 
$G = a_0+\sum_{j>0} a_j F_j$ where $F_j$ satisfy the induction hypothesis
and $a_j > 0$ for $j>0$.
The $\relu$ gate zeros out all the negative values of $G$.
By the sub-gradient sum property,
for all $x$,
$$\partial G (x) = \sum_{j>0} a_j \partial F_j(x).$$
It follows that (by (2) and (3) from \Cref{clm:sub-grad})

\begin{align*}
\partial F(x)
& \leq 
 \sum_{j>0} a_j \partial F_j(x) .
\end{align*}
The induction hypothesis 
and the fact that $a_j >0$ for $j>0$
imply
that the gradient of $F$ is non-negative. 
Now, let $x \leq y$.
If $G(x)<0$ then $\partial F(x) = \{0\}$ and isotonicity follows.
Otherwise, when $G(y)>0$, we have
\begin{align*}
\partial F(x)
 \leq  \sum_{j>0} a_j \partial F_j(x)  \leq
 \sum_{j>0} a_j \partial F_j(y)    = \partial F(y) . 
\end{align*}

Finally, when $G(x)=G(y) = 0$,
\begin{align*}
\partial F(x)
& = \conv \Big( \{0\} \cup
\Big( \sum_{j>0} a_j \partial F_j(x) \Big) \Big) \\
&  \leq
  \conv \Big( \{0\} \cup
\Big( \sum_{j>0} a_j \partial F_j(y) \Big) \Big) = \partial F(y). \qedhere
\end{align*}

\end{proof}

We shall use the following simple one-dimensional claim.

\begin{claim}
\label[claim]{clm:in-appx1}
Let $f: [0,1] \to \R$ be a convex $\CPWL$ function.
Let $a,b \in \R$.
For every $x \in [0,1]$, choose
$g_x \in \partial f(x)$.
Then, there is $x \in [0,1]$ so that
$$g_x  \leq a+ 8  \int_0^1 |f(y) - (a y+b) | dy $$
and there is $x' \in [0,1]$ so that
$$g_{x'}  \geq a- 8  \int_0^1 |f(y) - (a y+b)| dy .$$
\end{claim}

\begin{remark}
Convexity is not essential for the claim above. A $\CPWL$ function is differentiable almost everywhere, and we can replace $g_x$ by the gradient at $x$. 
\end{remark}

\begin{proof}
Observe that without loss
of generality $f(x_0) = a x_0 + b$ for some $x_0 \in [0,1]$ because otherwise we can reduce the integral by changing $b$.
If $x_0 \leq 1/2$, argue as follows.
Because $f$ is differentiable almost everywhere, for $x > x_0$ we can write
$$f(x) = f(x_0)+ \int_{x_0}^x g_y dy.$$
A similar statement holds for $ax+b$.  
Let 
$$c:=\int_{x_0}^1 (f(x) - (a x+b)) dx .$$
Write
\begin{align*}
c
& = \int_{x_0}^1 \int_{x_0}^x (g_y-a) dy dx \\
& = \int_{x_0}^1 (g_y-a) \int_{y}^1  dx dy \\
& = \int_{x_0}^1 (g_y-a) (1-y) dy .
\end{align*}

Assume without loss of generality that $c \geq 0$.
Because $f$ is $\CPWL$,
because $y\leq 1$ and because $\int_{x_0}^1 (1-y) dy \geq \frac{1}{8}$, 
if for all $y \in [x_0,1]$ we have $g_y-a > 8c$ then
$$c = \int_{x_0}^1 (g_y-a)(1-y)dy > 
8c \int_{x_0}^1 (1-y)dy \geq 
c,$$
which is a contradiction.

It follows that there exists $x$ so that $g_{x} -a \leq 8c$.
Similarly,
if for all $y \in [x_0,1]$ we have $g_y-a < 0$ then
$$c = \int_{x_0}^1 (g_y-a)(1-y)dy < 0 \leq c,$$
which is a contradiction.
It follows that there exists $x'$ so that $g_{x'} -a \geq 0 \geq - 8c$. The case $x_0 > 1/2$ is symmetric ($x \mapsto 1-x$).

\end{proof}

\begin{proof}[Proof of \Cref{thm:noApp}]
Let $F \in \CPWL_2$ be convex 

with isotonic gradient. 
Let
$$\eps:= \int_{x_1 \in [0,1]} \int_{x_2 \in [0,1]} |F(x) - \MAX_2(x)| dx_2 dx_1 .$$
It follows that
$$\int_{x_1 \in [1/4,1/2]} \int_{x_2 \in [0,1/4]} |F(x) - \MAX_2(x)| dx_2 dx_1 \leq \eps.$$

In this domain, $\MAX_2(x) = x_1$.
It follows 
that for some $x^*_2 \in [0,1/4]$,
$$ \int_{x_1 \in [1/4,1/2]} |F(x_1,x_2^*) - x_1|d  x_1  \leq 4 \eps.$$
For every $x_1 \in [1/4,1/2]$,
choose $g_{x_1} \in \partial F(x_1,x_2^*)$.
Let $f^*(x_1)=F(x_1,x_2^*)$. 
Notice that $(g_{x_1})_1$ is also a sub-gradient of $f^*$ at $x_1$. By \Cref{clm:in-appx1} applied to $f^*$ after rescaling $[1/4,1/2]$ to $[0,1]$,
there is $x^*_1 \in [1/4,1/2]$
so that $$g^*_1  \geq 1 - 512 \eps$$
where $x^*=(x_1^*,x_2^*)$ and $g^* = g_{x_1^*}$.

In the domain $[1/2,3/4] \times [3/4,1]$, 
the function $\MAX_2(x)$ is equal to $x_2$. For a fixed $x_2$, as a linear function in $x_1$, its slope is zero.
By isotonicity, we can choose the subgradients at all points in this domain to be coordinatewise at least $g^*$, since all these points are above $x^*$.
By a similar argument as above, there is $\tilde x$ in this domain with chosen subgradient $\tilde g \in \partial F(\tilde x)$
so that $${\tilde g}_1 \leq 0+512 \eps.$$

Finally, by our choice of subgradients,
\begin{equation*}
1 - 512 \eps \leq g^*_1 \leq {\tilde g}_1 \leq 512 \eps.
\end{equation*}
Thus $\eps \geq 1/1024$, so the theorem holds with constant $1/2048$.
\end{proof}

\section{Monotone depth lower bounds}
\label[appendix]{sec:Monotone}

In this section, we prove \Cref{thm:LB-f}; we show that
$$\f_n(x) = \relu(x_n + \f_{n-1}(x_1,\ldots,x_{n-1}))$$
requires monotone $\relu$ networks of depth at least $n$,
where $\f_0 = 0$. 
An analogous argument proves~\Cref{thm:MaxICNN}.
The proof is based on the observation that
$N(\f_n)$ is the $n$-dimensional simplex with vertex-set
$$V_n = \{0,e_n,e_n+e_{n-1},\ldots,e_n+e_{n-1}+\ldots+e_1\}$$
where $e_i$ is the $i$'th standard unit vector. This is important because the simplex is known to be indecomposable (as discussed above). 
These simplices belong to a type of simplices known as Schl\"afli orthoscheme; see
\Cref{fig:orthoscheme}.

We are ready to prove the depth lower bound.
The proof is by induction on $n$.
The case $n=1$ is trivial.
Assume that $n>1$ and that $\f_n \in \relu^+_{n,k}$.
It follows that $N(\f_n) \in \cP(\relu^+_{n,k})$.
Write
$$N(\f_n) = \Big( \sum_j a_j \conv (\{0\}\cup P_j)
\Big) + \Big( \sum_i a'_i P'_i \Big)$$
for $P_j \in \cP(\relu^+_{n,k-1})$, $P'_i \in \cP(\relu^+_{n,k-1})$ and $a_j,a'_i >0$;
each of the two sums may be empty.
Because $N(\f_n)$ is indecomposable,
it follows by induction that there is some polytope $P_*$ in $\cP(\relu^+_{n,k-1})$ so that $$N(\f_n) = \conv (\{0\} \cup P_*).$$
In particular, with $u=e_n$ we have
$$E:=N(\f_n) \cap H(N(\f_n),u) = P_* \cap H(P_*,u)$$
and $h(P_*,u) = 1$. 
The polytope $E-u$ is equal to $N(\f_{n-1})$
in the space~$u^{\perp}$.

For every polytope $P$ that appears in the construction of $P_*$,
replace $P$ by $(P \cap H(P,u)) - h(P,u) u \subset u^\perp$.
This replacement is defined inductively.
Points $p$ are orthogonally projected to $p' \in u^\perp$; the point $p'$ is still non-negative
in $u^\perp$.
If $P = \sum_j P_j$
and each $P_j$ was replaced by $P'_j$,
then replace $P$ by $P' = \sum_j P'_j$.
If $P = \conv (\{0\} \cup Q)$
and $Q$ was already replaced by $Q'$,
then replace $P$ by $\conv (\{0\} \cup Q')$ if $Q \subset u^\perp$,
and by $Q'$ otherwise.
By \Cref{clm:supp},
we can deduce that
the new $\relu^+_{n,k-1}$ network
computes $N(\f_{n-1})$.
By induction, we can deduce that $k \geq n$.

\newpage
\input{checklist.tex}

\end{document}

%% file: checklist.tex
\section*{NeurIPS Paper Checklist}

\begin{enumerate}

\item {\bf Claims}
    \item[] Question: Do the main claims made in the abstract and introduction accurately reflect the paper's contributions and scope?
    \item[] Answer: \answerYes{} 
    \item[] Justification: Our abstract lists clearly our contributions and main theorems and results. Proof sketches and detailed proofs are available in Sections 2--3 and Appendices A--D.

\item {\bf Limitations}
    \item[] Question: Does the paper discuss the limitations of the work performed by the authors?
    \item[] Answer: \answerYes{} 
    \item[] Justification: Section 4 discusses the limitations of our work and future work directions where more work is required.

\item {\bf Theory assumptions and proofs}
    \item[] Question: For each theoretical result, does the paper provide the full set of assumptions and a complete (and correct) proof?
    \item[] Answer: \answerYes{} 
    \item[] Justification: Each result is presented with clear definitions and complete proofs.

    \item {\bf Experimental result reproducibility}
    \item[] Question: Does the paper fully disclose all the information needed to reproduce the main experimental results of the paper to the extent that it affects the main claims and/or conclusions of the paper (regardless of whether the code and data are provided or not)?
    \item[] Answer: \answerNA{} 
    \item[] Justification: We do not have experiments.

\item {\bf Open access to data and code}
    \item[] Question: Does the paper provide open access to the data and code, with sufficient instructions to faithfully reproduce the main experimental results, as described in supplemental material?
    \item[] Answer: \answerNA{} 
    \item[] Justification: We do not have code to provide.

\item {\bf Experimental setting/details}
    \item[] Question: Does the paper specify all the training and test details (e.g., data splits, hyperparameters, how they were chosen, type of optimizer) necessary to understand the results?
    \item[] Answer: \answerNA{} 
    \item[] Justification: Our paper does not contain experiments.

\item {\bf Experiment statistical significance}
    \item[] Question: Does the paper report error bars suitably and correctly defined or other appropriate information about the statistical significance of the experiments?
    \item[] Answer: \answerNA{} 
    \item[] Justification: Our paper does not include experiments

\item {\bf Experiments compute resources}
    \item[] Question: For each experiment, does the paper provide sufficient information on the computer resources (type of compute workers, memory, time of execution) needed to reproduce the experiments?
    \item[] Answer: \answerNA{} 
    \item[] Justification: Our paper does not include experiments.
    
\item {\bf Code of ethics}
    \item[] Question: Does the research conducted in the paper conform, in every respect, with the NeurIPS Code of Ethics \url{https://neurips.cc/public/EthicsGuidelines}?
    \item[] Answer: \answerYes{} 
    \item[] Justification: We have reviewed and conformed to the code of ethics.

\item {\bf Broader impacts}
    \item[] Question: Does the paper discuss both potential positive societal impacts and negative societal impacts of the work performed?
    \item[] Answer: \answerNA{} 
    \item[] Justification: There is no clear immediate societal impact to our work.
    
\item {\bf Safeguards}
    \item[] Question: Does the paper describe safeguards that have been put in place for responsible release of data or models that have a high risk for misuse (e.g., pre-trained language models, image generators, or scraped datasets)?
    \item[] Answer: \answerNA{} 
    \item[] Justification: Our paper poses no such risks.

\item {\bf Licenses for existing assets}
    \item[] Question: Are the creators or original owners of assets (e.g., code, data, models), used in the paper, properly credited and are the license and terms of use explicitly mentioned and properly respected?
    \item[] Answer: \answerNA{} 
    \item[] Justification: We do not use any assets, existing or otherwise.

\item {\bf New assets}
    \item[] Question: Are new assets introduced in the paper well documented and is the documentation provided alongside the assets?
    \item[] Answer: \answerNA{} 
    \item[] Justification: Our paper does not release new assets.

\item {\bf Crowdsourcing and research with human subjects}
    \item[] Question: For crowdsourcing experiments and research with human subjects, does the paper include the full text of instructions given to participants and screenshots, if applicable, as well as details about compensation (if any)? 
    \item[] Answer: \answerNA{} 
    \item[] Justification: Our paper does not involve crowdsourcing nor research with human subjects.

\item {\bf Institutional review board (IRB) approvals or equivalent for research with human subjects}
    \item[] Question: Does the paper describe potential risks incurred by study participants, whether such risks were disclosed to the subjects, and whether Institutional Review Board (IRB) approvals (or an equivalent approval/review based on the requirements of your country or institution) were obtained?
    \item[] Answer: \answerNA{} 
    \item[] Justification: Our paper involves neither crowdsourcing nor research with human subjects.

\item {\bf Declaration of LLM usage}
    \item[] Question: Does the paper describe the usage of LLMs if it is an important, original, or non-standard component of the core methods in this research? Note that if the LLM is used only for writing, editing, or formatting purposes and does \emph{not} impact the core methodology, scientific rigor, or originality of the research, declaration is not required.
    \item[] Answer: \answerNA{} 
    \item[] Justification: LLMs are not a component of the core research methods.

\end{enumerate}